\documentclass[a4paper,10pt]{article}
\usepackage[utf8]{inputenc}
\usepackage{amsthm}
\usepackage{amsmath}
\usepackage{authblk}
\usepackage{mathtools}
\usepackage[mathscr]{euscript}
\usepackage{enumitem}
\usepackage{bm}

\usepackage{dsfont}
\newcommand{\Sb}{\mathbb{S}}
\newcommand{\Si}{\mathcal{S}}
\newcommand{\Pit}{\mathcal{P}}
\newcommand{\Rb}{\mathbb{R}}
\newcommand{\Cb}{\mathbb{C}}
\newcommand{\Ai}{\mathcal{A}}
\newcommand{\Ui}{\mathcal{U}}
\newcommand{\ct}{\overline{\theta}}
\newcommand{\losst}{\tilde{\nabla} \ell}
\newcommand{\loss}{\nabla_\theta \ell}
\newcommand{\act}{\textbf{\textit{act}}}
\newcommand\restr[2]{{
  \left.\kern-\nulldelimiterspace 
  #1 
  \vphantom{\big|} 
  \right|_{#2} 
  }}
\DeclareMathOperator*{\argmax}{argmax}

\usepackage{cite}
\usepackage{amssymb}
\usepackage{nicefrac}
\usepackage[pdftex]{graphicx}
\usepackage{mathrsfs}
\usepackage[scr=rsfs,cal=boondox]{mathalfa}
   \usepackage[justification=centering]{caption}
   \usepackage{float}
\usepackage{array}
\usepackage[ruled,vlined]{algorithm2e}
\newtheorem{theorem}{Theorem}

\newtheorem{lemma}{Lemma}

\newcommand{\ei}[1]{\boldsymbol{\mathcal{#1}}}

\title{ Deep Q-Learning: Theoretical Insights from an Asymptotic Analysis
}
\author{Arunselvan Ramaswamy\footnote{A. Ramaswamy is with the Heinz-Nixdorf Institute and the Department of Computer Science, Paderborn University, 33098 Paderborn, Germany (e-mail: arunr@mail.upb.de).} 
and Eyke H{\"u}llermeier\footnote{E. H{\"u}llermeier is with the Institute of Informatics at the University of Munich, 80538 Munich , Germany (e-mail: eyke@ifi.lmu.de).}}
\begin{document}
\maketitle

\begin{abstract}
Deep Q-Learning is an important reinforcement learning algorithm, which involves training a deep neural network, called Deep Q-Network (DQN), to approximate the well-known Q-function. Although wildly successful under laboratory conditions, serious gaps between theory and practice as well as a lack of formal guarantees prevent its use in the real world. Adopting a dynamical systems perspective, we provide a theoretical analysis of a popular version of Deep Q-Learning under realistic and verifiable assumptions. More specifically, we prove an important result on the convergence of the algorithm, characterizing the asymptotic behavior of the learning process. Our result sheds light on hitherto unexplained properties of the algorithm and helps understand empirical observations, such as performance inconsistencies even after training. Unlike previous theories, our analysis accommodates state Markov processes with multiple stationary distributions. In spite of the focus on Deep Q-Learning, we believe that our theory may be applied to understand other deep learning algorithms
\end{abstract}

\section{Introduction} \label{sec_introduction}

Reinforcement Learning (RL) is an important branch of machine learning, which has received increasing attention in the recent past. Roughly speaking, it considers an autonomous agent interacting with a dynamic environment, and seeks to learn a policy (prescribing actions depending on the current state of the environment) maximizing the agent's welfare in the course of time. A popular variant of RL, called \emph{Deep Reinforcement Learning} (DeepRL), combines the fundamental principles of RL with the power of deep learning. DeepRL has exhibited tremendous empirical success in recent years in wide ranging fields, from games \cite{silver2017mastering} to self-driving cars \cite{lillicrap2016continuous}.

In this paper, we focus on the popular DeepRL algorithm \emph{Deep Q-Learning}, which was introduced in \cite{mnih2015human} and shown to achieve superhuman performance in playing ATARI video games. Q-learning is a specific approach to RL, which focuses on learning the so-called Q-function to evaluate state-action pairs. 
In Deep Q-Learning, this function is represented by a deep neural network, called the Deep Q-Network (DQN), and learning the optimal Q-function is accomplished by minimizing the squared Bellman loss (error). DQN training typically involves repeated interactions with a simulator, or the use of historical data. In spite of its undoubted potential, Deep Q-Learning is still lacking a solid theoretical foundation. This also explains, at least partly, its slow adoption for real-world applications, although it generally performs well in a laboratory setting. The lack of a comprehensive understanding of the training process also hampers the explanation of empirical findings, such as suboptimal performance even when training is deemed sufficient. 


First theoretical results include sufficient conditions for convergence of Deep Q-Learning, provided the DQN uses rectified linear units as activation functions \cite{yang2019theoretical}. The analysis requires strict conditions on the Bellman operator and the distribution of state Markov process. In \cite{zou2019finite}, a non-asymptotic finite sample analysis of Deep Q-Learning with linear function approximation (instead of deep neural network (DNN) approximation) is presented. While studies like these focus on sufficient conditions for convergence, the focus in \cite{achiam2019towards} is on characterizing conditions under which Deep Q-Learning is divergent. Although understanding divergence is of paramount importance, assuming linearity of the function approximator reduces the applicability of such results in real-world scenarios. This is because, in practice, deep neural networks, which are non-linear functions, are used as function approximators. There are many recent theoretical results that are based on the linearity of the function approximator, see for e.g., \cite{xu2020finite}, \cite{du2020agnostic} and \cite{chen2019finite}. In \cite{cai2019provably}, the topic of efficient exploration in policy optimization is explored from a theoretical perspective. While these preliminary results are important and interesting, they do not immediately apply to Deep Q-Learning \emph{as implemented in practice}, due to unrealistic simplifications and restrictive assumptions.

\paragraph{Our contributions.}
The performance of Deep Q-Learning strongly depends on the training procedure. Empirically, it has been observed that performance is great in some test scenarios and poor in others. The hitherto available theory does not explain this phenomenon,  nor does it account for other empirical observations of similar kind. The main contribution of this paper is a comprehensive analysis of Deep Q-Learning that provides such explanations\,---\,under assumptions that are practical and verifiable. 

We show that the squared Bellman loss is minimized over the set of state-action pairs, distributed in accordance with a measure obtained as a limit of a natural measure process associated with the training procedure. We also show that this limiting measure is stationary with respect to the state Markov process. Further, its empirical estimate can be used to retrain and boost performance. As stated earlier, the limiting measure is strongly shaped by the training process. It is worth mentioning that, unlike previous literature, our analysis allows for multiple stationary distributions of the state Markov process. 

The most popular implementation of Deep Q-Learning involves the use of a target network. The use of such a network is shown to improve learning stability. However, it can be shown that the convergence properties of Deep Q-Learning does not change with the use a target network. Since, we focus on convergence in this paper, and not stability, we do not consider implementations with target networks. More importantly, it has recently been shown in \cite{kim2019deepmellow} that Deep Q-Learning that uses the ``mellowmax'' operator, instead of the usual ``max'' operator eliminates the need for target networks. They show superior performance as compared to traditional Deep Q-Learning with target networks, in many benchmark scenarios. Although, we do not explicitly consider the algorithm described \cite{kim2019deepmellow}, through appropriate modifications of the loss function our analysis can be extended to encompass this scenario as well.

Another popular implementation involves the use of a buffer memory called the \emph{experience replay}. It stores past experiences for relearning purposes. The main analysis presented in Sections~\ref{sec_dqn} and \ref{sec_conv} do not account for the use of an experience replay. However, in Section~\ref{sec_replay}, we discuss the steps involved in extending our analysis to account for this. {We show that experience replay affects the quality of performance by shaping the limiting distribution}. Additionally, it may aid in stabilizing the DQN training. 

For our analyses, we utilize tools from the fields of stochastic approximation algorithms (SA) \cite{borkar2009stochastic, kushner2003stochastic}, stochastic processes \cite{durrett2010probability}, measure theory \cite{billingsley2013convergence}, and viability theory \cite{aubin2012differential}.

\section{PRELIMINARIES}\label{sec_prelima}
For a fairly detailed introduction to reinforcement learning, the reader is referred to Appendix~\ref{sec:preliminaries_RL}. In what follows, we discuss the architecture of Deep Q-Network.
\subsection{Deep Q-Network (DQN)} \label{DQN_prelim}
Since a DQN is essentially an artificial neural network, or simply a neural network (NN), we begin by describing one. In particular, we discuss the architecture of a \emph{fully connected feedforward network} with real-valued vector inputs. \emph{Activation functions} form the basic building blocks of an NN. The typical domain for an activation function $\sigma$ is $\Rb$, and its range $\mathcal{R}$ is usually a subset of $\Rb$, i.e., $\sigma: \Rb \to \mathcal{R} \subset \Rb$. Depending on whether the range of $\sigma$, $\mathcal{R}$, is compact or unbounded, it is said to be \emph{squashing} or \emph{non-squashing}, respectively. There are many activation functions, the following are a few examples considered in this paper:
(a) Sigmoid $\left[\nicefrac{1}{1 + e^{-x}} \right],$ 
(b) Hyperbolic Tangent $\left[\nicefrac{e^{x} - e^{-x}}{e^{x} + e^{-x}} \right],$ 
(c) Gaussian Error Linear Unit $\left[ x \mbox{$\int\limits_{-\infty} ^x$} \nicefrac{e^{-y^2 / 2}}{\sqrt{2 \pi}} \ dy \right],$ and
(d) Sigmoid Linear Unit $\left[ \nicefrac{x}{1 + e^{-x}} \right]$.

An NN is a collection of \emph{activations} that are arranged in a sequence of \emph{layers}, starting with an \emph{input} layer, then followed by one or more \emph{hidden} layers, and ending with the \emph{output} layer. An NN with two or more hidden layers is called a \emph{Deep Neural Network} (DNN). Figure~\ref{DQN} illustrates one such NN architecture. By convention, an NN is constructed from left to right starting with the input layer and ending with the output layer. Further, the layers are arranged in a feedforward architecture, in that any two successive layers constitute a \emph{complete bipartite graph} with edges directed from the left layer into the right.

\begin{figure}[H]
\includegraphics[width=.48\textwidth]{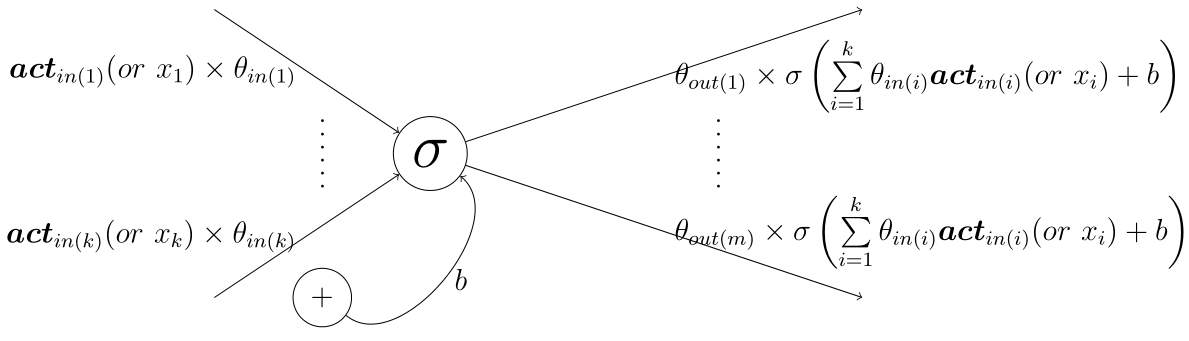}
\caption{Single activation from some layer}  \label{neuron}
\end{figure}

Figure~\ref{neuron} illustrates a single activation $\sigma$ within some layer. There are $k$ edges leading into and $m$ leading out of $\sigma$, where $m, k \ge 1$. When $\sigma$ is in the input layer, the in-edges connect the $k$ components of the input vector to it. As a part of other layers, the \emph{in-edges} connect the $k$ activation-outputs from the previous layer to its input. Further, each in-edge is associated with a weight that equals the product of the corresponding previous layer activation output $\act_{in(i)}$ (or input component $x_i$) and network-weight $\theta_{in(i)}$, $1 \le i \le k$. The input value to the activation is given by 
$$
\sum \limits_{i=1}^k \act_{in(i)} \theta_{in(i)} + b
$$
or $\sum_{i=1}^k x_i \theta_{in(i)} + b$, where $b$ is a tunable bias term. 
Suppose $\sigma$ is part of an input or hidden layer, then the edges leading out of it, the \emph{out-edges}, connect its output 
\begin{equation}\label{eq:output}
\sigma \left(\sum \limits_{i=1}^k \act_{in(i)} \theta_{in(i)} + b \right)
\end{equation} 
(or $\sigma (\sum_{i=1}^k x_i \theta_{in(i)} + b )$) to the input of the $m$ activations in the following layer. Finally, if $\sigma$ is part of the output layer, its output \eqref{eq:output}
is combined with the output from other activations that also belong to the outer layer, to obtain the required NN output. 
For more details the reader may refer to \cite{graupe2013principles, yegnanarayana2009artificial}.

\noindent
\emph{Note on tunable biases:} Subsequently, we assume that there are no tunable biases added to the activation inputs. In particular, we assume that the input is merely $\sum \theta_{in(i)} \act_{in(i)}$ $\text{(or } \sum \theta_{in(i)} x_i$ if the activation belongs to the input layer). We make this simplification for the sake of clarity in presentation. Our analysis will remain unaltered, except for minor bookkeeping, if one wishes to account for tunable biases.

\begin{figure}[h]
\begin{center}
\includegraphics[width=.48 \textwidth]{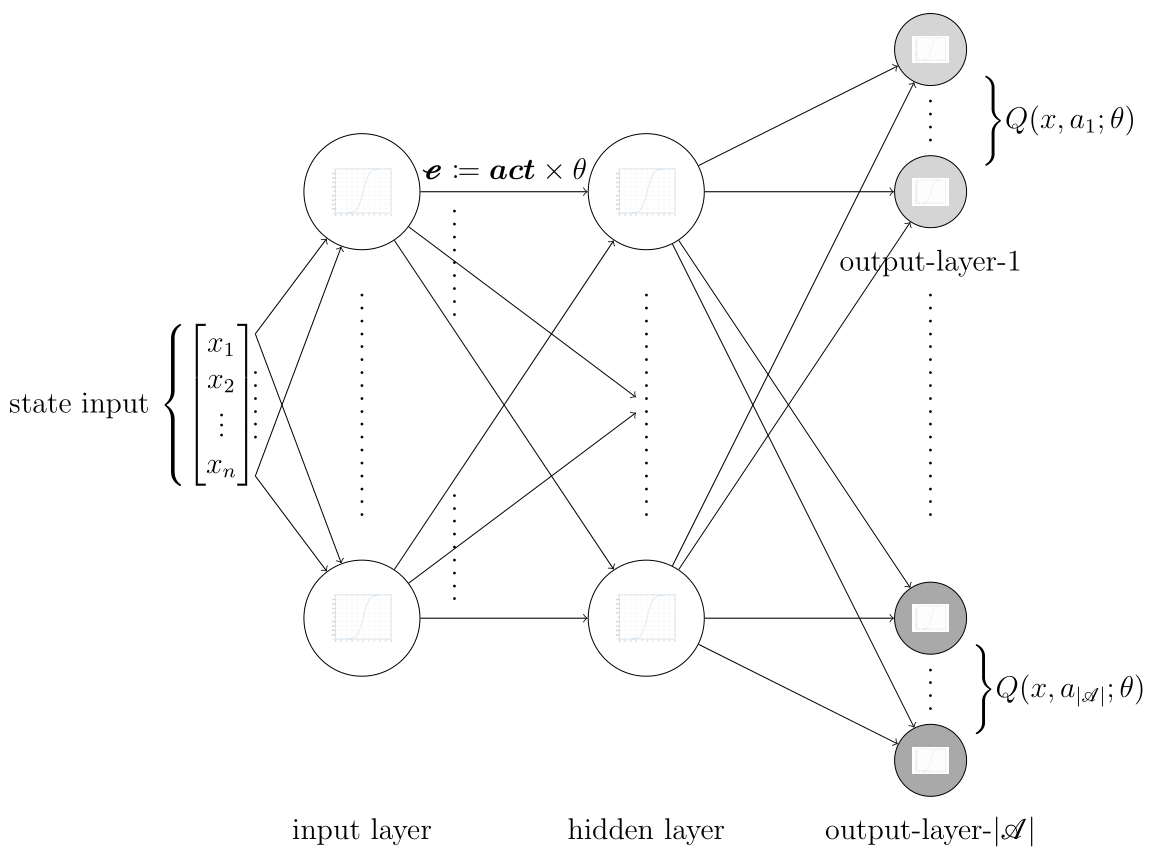}
\caption{Schematic Representation of a DQN}  \label{DQN}
\end{center}
\end{figure}

We are now ready to discuss the DQN architecture, also illustrated in Fig.~\ref{DQN}. Its input is the state vector $x \in \mathbb{S}$, and its output is a vector of dimension $| \mathcal{A} |$. The DQN output layer is a union of $| \mathcal{A} |$ separate (sub) output layers, one for each action. The output-layer-$i$ associated with action $a_i$, is fully connected to the previous hidden layer, see Fig.~\ref{DQN}. In particular, they are connected to the same layer. Let $l(a)$ be the number of activations in the output layer associated with action $a$, then $Q(x,a; \theta) \coloneqq \sum_{i=1}^{l(a)} \act_{a(i)} \theta_{a(i)}$, where $\act_{a(i)}$ is the activation-$i$ output and $\theta_{a(i)}$ is the associated network weight. \emph{Note that we use $\theta_{a(i)}$ and $\act_{a(i)}$, instead of merely using $\theta_i$ and $\act_i$, respectively, to emphasize the association with action.} 

In a nutshell, DQN is a parameterization of the vector $\left(Q^*(x,a)\right)_{a \in \mathcal{A}} $, where $Q^*$ is the optimal Q-function. In Deep Q-Learning, one updates the DQN weights $\theta \coloneqq \left(\theta_e \mid e \text{ is an edge in  the DQN} \right)$ iteratively, in order to find $\theta^*$ such that $Q(x,a; \theta^*) \approx Q^*(x,a)$, $\forall \ (x,a)  \in \Sb \times \Ai.$

\section{DEEP Q-LEARNING} \label{sec_dqn}
To minimize the squared Bellman loss, Deep Q-learning iterates the update 
\begin{equation}\label{dqn_1}
\theta_{n+1} \leftarrow \theta_n + \gamma(n) \nabla_\theta \ell(\theta_n, x_n, a_n) 
\end{equation}
of the DQN weight vector $\theta \in \mathbb{R}^d$, where the following notation is used:
\begin{enumerate}[label=(\roman*)]
\item $\theta_n \in \mathbb{R}^d$, $x_n \in \mathbb{S}$, and $a_n \in \Ai$ for $n \in \mathbb{N}_0$. The state space $\mathbb{S}$ is assumed to be $\mathbb{R}^n$ for some $n \ge1$, and $\Ai$ is a finite set of actions. 
\item The loss gradient of \eqref{dqn_1} 
is given by 
\begin{align}\label{eq:lgr}
& \nabla_\theta \ell(\theta_n, x_n, a_n)  =  \nabla_\theta Q(x_n,  a_n; \theta_n) \times \\
& \quad (r(x_n , a_n) + \alpha \max \limits_{a' \in \mathcal{A}} Q(x_{n+1}, a'; \theta_n) -  Q(x_n,  a_n; \theta_n) )  \nonumber \, ,
\end{align}
where $\alpha$ is the discount factor. Since $a_n$ is the action taken at time $n$, $\nabla_\theta \ell(\theta_n,  x_n, a_n)$ denotes the loss-gradient back-propagated via $a_n$.
\item $\gamma(n)$ is the step size sequence satisfying the standard assumptions of non-summability and square summability.
\end{enumerate}

Note that the loss gradient is calculated using the \emph{sample value} $\max \limits_{a' \in \mathcal{A}} Q(x_{n+1}, a'; \theta_n)$ instead of the \emph{expected value} $\int \max \limits_{a' \in \mathcal{A}} Q(x', a'; \theta_n) p(dx' \mid x,a)$. This is because the transition kernel $p$ is unknown in real applications. The algorithm \emph{observes} the next state $x_{n+1}$ and the reward $r(x_n,a_n)$, after applying $a_n$ in state $x_n$.

The state Markov process is determined by the transition kernel $p(dy \mid x, a)$. In training, actions are picked through a policy that \emph{exploits} the approximation capability of DQN, while simultaneously \emph{exploring} new actions. In other words, the transition kernel is indirectly influenced by the network weights. Hence, we denote the controlled transition kernel by $p(dy \mid x, a, \theta)$. For fixed weights $\theta$ and a fixed stochastic policy $\pi_\theta$, the transition kernel is given by
$$
\tilde{p}_\theta(dy \mid x) = \sum \limits_{a \in \Ai} p(dy \mid x,a, \theta) \pi_\theta (x, da) \, .
$$
The policy is subscripted with $\theta$ to emphasize that it depends on the network weights (via exploitation). Let us suppose that $\pi_\theta$ only exploits and does not explore. Then the above stochastic policy is the Dirac measure given by $\pi_\theta(x, da) = \delta_{\argmax \limits_{a \in \Ai} }Q(x, a; \theta)$. Furthermore, the above kernel becomes
$$
\tilde{p}_\theta(dy \mid x) = p \left(dy \bigm\vert x,\argmax \limits_{a \in \Ai} Q(x, a; \theta),  \theta \right) \, .
$$

\noindent
\emph{Note on notation:} We use $dy$ and $da$ (instead of just $y$ and $a$) to represent the variables on which $p$ and $\pi_\theta$, respectively, define distributions, so as to easily distinguish the variable under consideration.
Suppose $\pi_\theta$ is a Dirac measure. Then, through a slight abuse of notation, we use $\pi_\theta (x)$ to represent $\argmax \limits_{a \in \Ai} Q(x, a; \theta)$.

\subsection{Assumptions}
The assumptions required to analyze \eqref{dqn_1} are as follows:

\begin{itemize}
\item[(A1)] $\gamma(n) > 0$ for all $n \ge 0$, $\sum_{n \ge 0} \gamma(n) = \infty$ and $\sum_{n \ge 0} \gamma(n)^2 < \infty$. Further, the sequence monotonically decreasing.
\item[(A2)] (a) $\sup_{n \ge 0} \lVert \theta_n \rVert_2 < \infty$ a.s., (b) $\sup_{n \ge 0} \lVert x_n \rVert_2 < \infty$ a.s.
\item[(A3)] The state transition kernel $p(\cdotp \mid x, a)$ is continuous in the $x$-coordinate.
\item[(A4)] The DQN is composed of activation functions that are squashing and twice continuously differentiable.
\item[(A5)] The reward function $r: \Sb \times \Ai \to \Rb$ is continuous.
\end{itemize}

The first assumption regarding the step size sequence (learning rate) is standard in the literature. Recall that the loss gradient in \eqref{dqn_1} is calculated using samples that are supposed to approximate expected values. The resulting sampling errors are controlled using step sizes that are square summable. The stability assumption (A2) is essential for analyzing the long-term behavior of \eqref{dqn_1}. 

Consider two different but ``closely neighbored'' states in the environment. Assumptions (A3) and (A5) state that the consequences (successor states and rewards, respectively) of taking the same action in these states are similar. These assumptions are not only natural, but also ensure the performance of approximation-based algorithms like Deep Q-learning. As long as the state-action pairs encountered during training are a rich enough representation of $\Sb \times \Ai$, (A3) and (A5) facilitate good approximation of the Q-function. 


The assumption of squashing activations (A4) is mainly made for the sake of clarity of presentation and can easily be relaxed.  
An extension to general (twice continuously differentiable) activations is provided in Section~\ref{sec_activation}.

\subsection{Properties of the loss gradient} \label{sec_grad_prop}
The aim of this section is to prove certain useful properties that facilitate an abstract view of the loss gradient, with lesser ``moving parts''. In particular,  we show that $\loss$ is (A) locally Lipschitz continuous in the $\theta$-coordinate, and (B) continuous in the $x$ and $a$-coordinates. Suppose we equip $\Ai$ with the discrete topology. Then, since $\Ai$ is a finite set, the resulting discrete space is compact, so that $\loss$ is trivially continuous in the $a$-coordinate. As for the rest, we relegate a couple of technical lemmata to Appendix~\ref{sec:aux_lemmas}, and summarize the required results in Lemma~\ref{grad_loss_continuity} below.

\noindent
Let us define the sequence $\{M_n\}_{n \ge 0}$ as follows: 
\begin{equation*}
\begin{split}
& M_{n} := \sum \limits_{m=0}^{n-1} \gamma(m) \psi_m, \ n \ge 0, \text{ where } \\
& \psi_m \coloneqq  \alpha [ \max \limits_{a \in \Ai} Q(x_{m+1}, a; \theta_m) - \\ 
& \quad \int \max \limits_{a \in \Ai} Q(x,a; \theta_m) p(dx \mid x_{m}, a_m, \theta_m) ] \nabla_\theta Q(x_m, a_m; \theta_m) \, . 
\end{split}
\end{equation*} 
It can be shown that $\{M_n\}_{n \ge 0}$ is a zero-mean martingale with respect to the filtration $\mathcal{F}_{n-1} \coloneqq \sigma \langle x_m, a_m, \theta_m \mid m \le n \rangle$, $n \ge 1$. Recall that we assume stability of \eqref{dqn_1} and the state sequence, i.e., 
$\sup_{n \ge 0} \ \lVert \theta_n \rVert < \infty$ and $\sup_{n \ge 0} \ \lVert x_n \rVert < \infty$ a.s. This, together with the twice continuous differentiability of $Q$ in the $\theta$-coordinate (shown in Lemma~\ref{grad_2cd}, Appendix~\ref{sec:aux_lemmas}), lets us conclude that $\sup_{n \ge 0} \left| Q(x_n, a_n; \theta_n) \right| < K_1 < \infty$, and that $\left\lVert \nabla_\theta Q(x_n, a_n; \theta_n) \right\rVert < K_2 < \infty$, where $K_1$ and $K_2$ are possibly sample-path dependent. Hence $ \sup_{n \ge 0} \left\lVert \psi_n \right\rVert \le K < \infty$, where $K$ may again be sample-path dependent. Finally, the square summability of the step size sequence, assumption (A1), implies that $\sum_{m=0}^n \gamma(m)^2 \lVert M_n \rVert^2 < \infty$ a.s. Convergence of the martingale sequence $\{M_n\}_{n \ge 0}$ follows from the martingale convergence theorem, see \cite{durrett2010probability}.

Recall the loss gradient $\nabla_\theta \ell(\theta_n, x_n, a_n)$ given by \eqref{eq:lgr}, 
%
and let us rewrite it using the definition of $\psi_n$ as
\begin{equation} \label{grad_loss_1}
\begin{split}
& \nabla_\theta \ell(\theta_n, x_n, a_n) =  \big(r(x_n , a_n) + \\ 
& \qquad \alpha \int \max \limits_{a' \in \mathcal{A}} Q(y, a'; \theta_n) p(dy \mid x_n, a_n, \theta_n) \\  
& \qquad - Q(x_n,  a_n; \theta_n) \big)  \nabla_\theta Q(x_n,  a_n; \theta_n)  + \psi_n .
\end{split}
\end{equation}
Hence, \eqref{dqn_1} becomes
\begin{equation} \label{dqn_2}
\theta_{n+1} \leftarrow \theta_n + \gamma(n) \left[ \nabla_\theta \hat{\ell}(\theta_n, x_n, a_n)  + \psi_n \right] \, ,
\end{equation}
\begin{equation*}
\begin{split}
\text{where }\nabla_\theta \hat{\ell}(\theta_n, x_n, a_n) \coloneqq  (r(x_n , a_n) + \\ \alpha \int \max \limits_{a' \in \mathcal{A}} Q(y, a'; \theta_n) p(dy \mid x_n, a_n, \theta_n) \\ - Q(x_n,  a_n; \theta_n) )  \nabla_\theta Q(x_n,  a_n; \theta_n).
\end{split}
\end{equation*}
Since the martingale sequence $\{M_n\}_{n \ge 0}$ converges a.s., the impact of $\psi_n$ vanishes asymptotically. In other words, \eqref{dqn_1} and \eqref{dqn_2} are asymptotically identical to (have the same limiting set as)
\begin{equation} \label{dqn_3}
\theta_{n+1} \leftarrow \theta_n + \gamma(n) \left[ \nabla_\theta \hat{\ell}(\theta_n, x_n, a_n) \right].
\end{equation}
\emph{Note on notation:} Rather than keeping track of two versions of the loss gradients, $\nabla_\theta \ell$ and $\nabla_\theta \hat{\ell}$ from equations \eqref{dqn_1} and \eqref{dqn_3}, respectively, we redefine $\nabla_\theta \ell \coloneqq \nabla_\theta \hat{\ell}$. With this slight abuse of notation, we hope to avoid unnecessary confusion. The reader does not need to track two different losses. In our subsequent analysis, when we refer to \eqref{dqn_1}, the associated loss gradient is
\begin{equation} \label{grad_loss_new}
\begin{split}
\nabla_\theta \ell(\theta_n, x_n, a_n) \coloneqq  (r(x_n , a_n) + \alpha \\ \int \max \limits_{a' \in \mathcal{A}} Q(y, a'; \theta_n) p(dy \mid x_n, a_n, \theta_n)  - \\ Q(x_n,  a_n; \theta_n) )  \nabla_\theta Q(x_n,  a_n; \theta_n).
\end{split}
\end{equation}
\normalsize

\begin{lemma} \label{grad_loss_continuity}
$\nabla_\theta \ell(\theta_n, x_n, a_n)$, redefined as \eqref{grad_loss_new}, is continuous and locally Lipschitz continuous in the $\theta$-coordinate.
\end{lemma}
\begin{proof}
For the proof, one can combine the consequences of (i) Lemmas~\ref{gp_bound}, \ref{grad_2cd} and \ref{grad_Qexp} (see Appendix~\ref{sec:aux_lemmas}), (ii) assumption (A5), i.e., the continuity of the reward function $r$, and (iii) the fact that the sum and product of continuous and locally Lipschitz continuous functions are also continuous and locally Lipschitz continuous, respectively.
\end{proof}

The Lipschitz constant from the above statement is local and changes with $\theta$.
However, as discussed before, following the proof of Lemma~\ref{grad_2cd} (see Appendix~\ref{sec:aux_lemmas}), it also depends on $x$. If the domain of a locally Lipschitz continuous function is restricted to a compact subset, then the restricted function is Lipschitz continuous. Assumption (A2) states that $\sup_{n \ge 0} \ \lVert \theta_n \rVert_2 < \infty$ and $\sup_{n \ge 0} \ \lVert x_n \rVert_2 < \infty$ a.s. This can be used to conclude that $\loss$ is Lipschitz continuous in the $\theta$-coordinate, when restricted to an appropriate compact subset of $\Rb^d \times \Sb$. We note that the Lipschitz constant may be sample-path dependent and refer to the proof of Lemma~1 in \cite{ramaswamy2019optimization}, where something very similar is shown.
\section{CONVERGENCE ANALYSIS} \label{sec_conv}
To analyze the long-term behavior of \eqref{dqn_1}, we first construct an associated continuous-time trajectory with identical limiting behavior. Then, instead of \eqref{dqn_1}, we may analyze the continuous-time trajectory. 

First, we divide the time axis $[0, \infty)$ using the given step size sequence as follows:
$$t_n = 0 \text{ and } t_n = \sum_{m=0}^{n-1} \gamma(m) \text{ for } n \ge 1.
$$
We now define the required trajectory $\ct \in C([0, \infty), \mathbb{R}^d)$ as follows:
\begin{enumerate}[label=(\alph*)]
\item $\ct(t_n) = \ct_n$,  $n \ge 0$,
\item $\ct(t) = \ct(t_n) + \frac{t - t_n}{t_{n+1} - t_n} \left[ \ct(t_{n+1}) - \ct(t_n) \right]$ for $t \in (t_n, t_{n+1})$ and $n \ge 0$.
\end{enumerate}

As the sequence of actions taken are directly linked to the DQN-weights $\theta$ via ``exploitation'', we need to better understand them. To this end, we define the following measure process:
$$
\mu(t) = \delta_{(x_n, a_n)}, \ t \in [t_n, t_{n+1}),
$$
where $\delta_{(x,a)}$ is the Dirac measure that places mass $1$ on the state-action pair $(x,a) \in \Sb \times \Ai$.
Hence $\mu: [0, \infty) \to \Pit (\Sb, \Ai)$ defines a process of probability measures on $\Sb \times \Ai$. For our analysis, we need to define limits for the ``left-shifted'' measure process $\{\mu([t_n, \infty))\}_{n \ge 0}$. For that purpose, we first define a metric space (similar to the one from \cite{borkar2006stochastic}) consisting of such measure processes below. 

To start with, we observe that the action space $\Ai$ is \emph{compact metrizable}, as it is discrete and finite. As for $\Sb$, recall our assumption $\Sb = \Rb^n$. Thus, it follows from the Alexandroff extension that $\Sb$ is one-point compactifiable. In particular, the inverse stereographic projection $S^{-1}: \Sb \to \Si^n$ is such that $\Si^n \setminus S^{-1}(\Sb) = (0, \ldots, 0, 1)$, where $\Si^n$ represents the $(n+1)$-dimensional Hausdorff compact sphere of radius $1$ centered at the origin, and $(0, \ldots, 0, 1)$ is the ``north pole''. In other words, the inverse stereographic projection is the required compactification embedding of $\Sb$ into $\Si^n$, see \cite{munkres}. 

Every measure $\nu \in \Pit(\Sb \times \Ai)$ has a push forward counterpart in $\Pit(\Si ^n \times \Ai)$. It  places mass $0$ on $(0,\ldots, 0, 1) \times \Ai$. \emph{Moving forward, note that we shall use the same symbol to represent both the measure and its push forward counterpart}. Also note that $\Si ^n \times \Ai$ is \emph{compact Hausdorff in the product topology}.



Let us define $\Ui$ to be the space of all measurable functions $\nu(\cdotp) = \nu(\cdotp, dx, da)$ from $[0, \infty)$ to $\Pit (\Si ^n \times \Ai)$.

\begin{lemma}
$\Ui$ is compact metrizable. Further, this metric coincides with the coarsest topology that renders continuous the map $$
\nu \mapsto \int \limits_0 ^T g(t) \int f d\nu(t) \ dt,
$$
for all, $T > 0$, $f \in \mathbb{C}(\Si ^n \times \Ai)$ and $g \in \mathbb{L}^2 ([0,T], \Rb)$. 
\end{lemma}
\begin{proof}
By emulating the proof of Lemma~3 in \cite{borkar2006stochastic} with $`` \Si ^n \times \Ai"$ replacing $`` \overline{S} "$, and making appropriate modifications, the required proof is obtained. We do not repeat it here, to avoid redundancies.
\end{proof}

Define $\losst (\theta, \nu) \coloneqq \int \loss (\theta, x, a)\ \nu(dx, da)$, where $\nu \in \Pit (\Sb, \Ai)$. Lemma~\ref{grad_loss_continuity} implies that \emph{$\losst$ is continuous in both coordinates and locally Lipschitz continuous in the $\theta$-coordinate}. Further, $ \lVert \losst (\theta, \nu) \rVert \le K (1 + \lVert \theta \rVert)$, i.e., its growth is bounded as a function of $\theta$ alone. Let us also define the following sequence of trajectories in $\Cb([0, \infty), \mathbb{R}^d)$: $\theta^n (t)  = \ct(t_n) + \int_{0}^t \losst (\theta^n(s), \mu^n (s)) \ ds$, where $\mu^n(t) \coloneqq \mu(t_n + t)$, $t \ge 0$ and $n \ge 0$. In other words, we consider solutions to the set of non-autonomous ordinary differential equations: $\left\{\dot{\theta}^n(t) = \losst(\theta^n(t), \mu^n(t)) \right\}_{n \ge 0}$. As stated earlier, to understand the long-term behavior of \eqref{dqn_1}, one can study the behavior of the limit of sequence $\{\ct([t_n, \infty))\}_{n \ge 0}$, in $\Cb([0, \infty), \mathbb{R}^d)$ as $n \to \infty$. We can show the following property.
\begin{lemma} \label{lemma:trajectory_track}
For every $T > 0$, 
$$ 
\lim \limits_{n \to \infty} \sup \limits_{t \in [0,T]} \lVert \ct(t_n + t) - \theta^n(t) \rVert = 0 \enspace .
$$
\end{lemma} 
\begin{proof}
Please refer to Appendix~\ref{sec:trajectory_track} for a proof.
\end{proof}
 
Then, instead of \eqref{dqn_1} or the associated trajectory $\overline{\theta}$, we could focus on the sequence of trajectories $\{ \theta^n([0, \infty))\}_{n \ge 0}$. Now we may tap into the rich literature of tools and techniques available from viability theory \cite{aubin2012differential, aubin2011viability}.

The family of trajectories $\{\theta^n([0, \infty))\}_{n \ge 0}$, in $\Cb([0, \infty), \Rb ^d)$, is equicontinuous and point-wise bounded. It follows from the Arzela-Ascoli theorem  \cite{billingsley2013convergence} that it is sequentially compact. Note that the topology of $\Cb([0, \infty), \Rb ^d)$ is the one induced by the topologies of $\Cb([0, T], \Rb ^d)$ for every $0 < T <  \infty$. Now, let us consider the family $\{\mu^n\}_{n \ge 0} \subset \Ui$. As $\Ui$ is a compact metric space, $\{\mu^n\}_{n \ge 0}$ is sequentially compact. Hence, there is a common subsequence $\{m(n)\} \subset \{n\}$ such that $\mu^{m(n)} \to \mu^\infty$ in $\Ui$ and $\theta^{m(n)} \to \theta^\infty$ in $\Cb([0, \infty), \Rb ^d)$. \emph{With a slight abuse of notation, we have $\mu^{n} \to \mu^\infty$ in $\Ui$ and $\theta^{n} \to \theta^\infty$ in $\Cb([0, \infty), \Rb ^d)$.} In other words, the sequences $\mu^n$ and $\theta^n$ are convergent in their respective spaces. 

Below we state another important result, namely that convergence of the measure process in $\Ui$ implies convergence in distribution of the corresponding measure sequence, at every point in time.
\begin{lemma} \label{ana_mu}
If $\mu^n \to \mu^\infty$ in $\Ui$, then a.e.\ $\mu^n(t) \to \mu^\infty(t)$ in $\Pit(\Sb \times \Ai)$ for $t\in [0, \infty)$.
\end{lemma}
\begin{proof}
We begin by recalling that the same notation is used to denote a measure on $\Sb \times \Ai$ and its push forward counterpart on $\Si ^n \times \Ai$. It follows from the definition of convergence in $\Ui$ that
$\int \limits_0 ^T g(s) \int f \mu^n (s, dx,da) \ ds \to \int \limits_0 ^T g(s) \int f \mu (s, dx,da) \ ds$ as $n \to \infty$, for every $g \in \mathbb{L}^2([0,T], \Rb)$ and $f \in \Cb(\Si ^n \times \Ai)$. We claim that this implies, for every $f \in \Cb(\Si ^n \times \Ai)$, $\int f \mu^n (s, dx,da) \to \int f \mu (s, dx,da)$ a.e. for  $s \in [0, \infty)$. Once this claim is proven, we can conclude that $s$-a.e.\ $\mu^n(s) \to \mu^\infty(s)$ in $\Pit(\Si ^n \times \Ai)$, which finally yields the lemma. 

To prove the claim, let us assume the contrary. In particular, we assume $\exists f \in  \Cb(\Si ^n \times \Ai)$, $T> 0$, $\epsilon > 0$ and a non-zero Lebesgue measure set $A \in \mathcal{B}([0,T]) $, such that at least one of the following properties holds for all $s \in A$:
\begin{enumerate}[label=(\alph*)]
\item $\liminf \limits_{n \to \infty}  \int f \mu^n (s, dx,da) - \int f \mu (s, dx,da)   > \epsilon$, 
\item $\liminf \limits_{n \to \infty}  \int f \mu^n (s, dx,da) - \int f \mu (s, dx,da)   < - \epsilon$, 
\item $\limsup \limits_{n \to \infty}  \int f \mu^n (s, dx,da) - \int f \mu (s, dx,da)   > \epsilon$, 
\item $\limsup \limits_{n \to \infty}  \int f \mu^n (s, dx,da) - \int f \mu (s, dx,da)   < -\epsilon$. 
\end{enumerate}
We only present arguments for case (a), as the corresponding ones for the others are identical. Since $f$ is bounded, we apply the Dominated Convergence Theorem (DCT) \cite{durrett2010probability} to conclude that
\begin{equation*}
\begin{split}
\liminf \limits_{n \to \infty} \int \limits_0 ^T \mathds{1}_{A} \left[ \int f \mu^n (s, dx,da) - \int f \mu (s, dx,da) \right] ds \\  > \epsilon \ l(A) > 0,
\end{split}
\end{equation*}
where $l(A)$ denotes the Lebesgue measure of $A$. This directly contradicts the definition of convergence of measures in $\Ui$. 

It is left to show that $\mu^n(t) \to \mu^\infty(t)$ in $\Pit(\Sb \times \Ai)$ a.e. for $t \in [0, \infty)$. To do this, we pick $t \in [0, \infty)$ such that $\mu^n(t) \to \mu^\infty(t)$ in $\Pit(\Si^n \times \Ai)$ and show that their pull back versions converge in  $\Pit(\Sb \times \Ai)$. This is done by showing that $\limsup \limits_{n \to \infty} \mu^n(t, C) \le \mu^\infty(t, C)$ for every closed set $C \in \mathcal{B}(\Sb \times \Ai)$ (Portmanteau theorem \cite{billingsley2013convergence}). 

We first observe that the measures $\{\mu^n(t)\}_{0 \le n \le \infty}$ are tight as a consequence of (A2). Hence they place a mass of $0$ on $(0,\ldots,0, 1) \times \Ai$. If we restrict these measures to $S^{-1}(\Sb) \times \Ai$, then $\restr{\mu^n}{S^{-1}(\Sb) \times \Ai} \overset{\text{d}}{\implies} \restr{\mu^\infty}{S^{-1}(\Sb) \times \Ai}$. Next, we consider an arbitrary closed subset $C \in \mathcal{B}(\Sb \times \Ai)$. Since the stereographic projection is bicontinuous, $\hat{C} \coloneqq \{(S^{-1}(x), a) \mid (x,a) \in C\}$ is closed in $S^{-1}(\Sb) \times \Ai$, equipped with subspace topology (with respect to $\Si^n \times \Ai$). Clearly, $ \limsup \limits_{n \to \infty} \mu^n(t, \hat{C}) \le \mu^\infty(t, \hat{C})$. Now, as $\mu^n(t, \hat{C})$ is the push forward measure of $\mu^n(t, C)$ for all $0 \le n \le \infty$, we obtain the required result.
\end{proof}
We can use one of the many available measurable selection theorems \cite{wagner1977survey} to drop the a.e.\ clause in the statement of Lemma~\ref{ana_mu}. Hence, we have hitherto shown that $\theta^n \to \theta^\infty$ in $\Cb ([0, \infty), \Rb^d)$ and $\mu^n(s) \overset{\text{d}}{\implies} \mu^\infty(s)$ for all $s \in [0, \infty)$. We now need to show that $\theta^\infty$ is a solution of $\dot{\theta}(t) = \losst (\theta(t), \mu^\infty (t))$. Then, one can study the limiting behavior of a solution to the ODE $\dot{\theta}(t) = \losst (\theta(t), \mu^\infty (t))$, to understand the long-term behavior of Deep Q-Learning given by \eqref{dqn_1}.
\begin{lemma} \label{ana_ode}
$\theta^\infty$ is a solution to $\dot{\theta}(t)  = \losst(\theta(t), \mu^\infty (t))$.
\end{lemma}
\begin{proof}
Fix an arbitrary $T > 0$. We need to show that $$\sup \limits_{t \in [0,T]} \left\lVert \theta^n(t) - \theta^\infty(0) - \int \limits_{0}^t \losst(\theta^\infty(s), \mu^\infty(s)) \ ds \right\rVert \to 0. $$
Let us first consider the following:
\begin{equation}
\label{ana_eq_main}
\begin{split}
\left\lVert \theta^n(0) + \int \limits_{0}^t \losst(\theta^n(s), \mu^n(s)) \ ds - \theta^\infty(0) - \right. \\ \left. \int \limits_{0}^t \losst(\theta^\infty(s), \mu^\infty(s)) \ ds \right\rVert,
\end{split}
\end{equation}
\begin{multline}\label{ana_eq_l31}
\left\lVert \theta^n(0) - \theta^\infty(0) \right\rVert +   \left\lVert \int \limits_{0}^t \losst(\theta^n(s), \mu^n(s)) \ ds -  \right. \\ \left. \int \limits_{0}^t \losst(\theta^\infty(s), \mu^n(s)) \ ds\right\rVert +   \left\lVert \int \limits_{0}^t \losst(\theta^\infty(s), \mu^n(s)) \ ds - \right. \\ \left. \int \limits_{0}^t \losst(\theta^\infty(s), \mu^\infty(s)) \ ds \right\rVert.
\end{multline}
Next, we note the following:
\begin{enumerate}[label=(\Alph*)]
\item From Lemma~\ref{ana_mu} we have $\mu^n(s) \overset{\text{d}}{\implies} \mu^\infty (s)$ (converges in distribution on $\Sb \times \Ai$) for all $s \in [0,T]$.
\item From (A2), i.e., the stability of the algorithm, and the boundedness of $\loss$ as a function of $\theta$, we get $\loss(\theta^\infty(s), \cdotp)  \in \Cb_b (\Sb \times \Ai)$. Hence, as a consequence of note~(A), $\int \loss(\theta^\infty(s), x, a) \mu^n(s) \to \int \loss(\theta^\infty(s), x, a) \mu^\infty(s)$ for all $s \in [0,T]$.
\end{enumerate}
Using DCT, we get 
\begin{equation} \label{ana_eq_l34}
\begin{split}
  \left\lVert \int \limits_{0}^t \losst(\theta^\infty(s), \mu^n(s)) \ ds -  \int \limits_{0}^t \losst(\theta^\infty(s), \mu^\infty(s)) \ ds \right\rVert  \\ \to 0.
  \end{split}
\end{equation}
Further, it follows from the Arzela-Ascoli theorem that the convergence in \eqref{ana_eq_l34} is uniform over $[0,T]$. 

Since $\losst$ is locally Lipschitz continuous in $\theta$, we get
\begin{equation} \label{ana_eq_l32}
\begin{split}
\left\lVert \int \limits_{0}^t \losst(\theta^n(s), \mu^n(s)) \ ds - \int \limits_{0}^t \losst(\theta^\infty(s), \mu^n(s)) \ ds\right\rVert \\ \le
L \int  \limits_{0}^t \left\lVert \theta^n(s) - \theta^\infty(s) \right\rVert \ ds.
\end{split}
\end{equation}
As $\theta^n \to \theta^\infty$ uniformly over $[0,T]$, the l.h.s.\ of $\eqref{ana_eq_l32} \to 0$ uniformly over $[0,T]$. The discussion surrounding \eqref{ana_eq_l34} and \eqref{ana_eq_l32} implies that $\eqref{ana_eq_l31} \to 0$ and hence $\eqref{ana_eq_main} \to 0$, uniformly over $[0,T]$. As $T$ is \emph{arbitrary}, the lemma follows.
\end{proof}

To develop a better understanding of Deep Q-Learning, we need to study $\mu^\infty$, the limiting distribution over the state-action pairs. In the following lemma, we show that $\mu^\infty(t, dx \times \Ai)$ is stationary with respect to the state Markov process, $\forall \ t \ge 0$. Recall that $p(\cdotp \mid x,a,\theta)$ is the controlled transition kernel of the state Markov process. We use $p(\cdotp \mid x,\Ai,\theta)$ to denote the probability associated with transitioning out of state $x$ (when some action is picked). We use $p(dy \mid x, \Ai, \theta) \ \mu(dx \times \Ai)$ to denote $\int _\Ai p(dy \mid x, a, \theta) \ \mu(dx , da)$. In words, it represents the probability to transition from state $x$ to state $y$, given that $(x,a) \sim \mu$. 

\begin{lemma} \label{ana_marginal_lemma}
For all $t \in [0, \infty)$, 
$\mu^\infty (t, dy \times \Ai) = \int_{\Sb} p(dy \mid x, \Ai, \theta^\infty(t)) \ \mu^\infty (t, dx \times \Ai)$. In other words, the limiting marginal constitutes a stationary distribution over the state Markov process. Further, $\{\mu^\infty(t, dx , da )\}_{t \ge 0}$ is tight.
\end{lemma}
For a proof of this lemma, we refer to Appendix~\ref{sec_marginal_lemma}.

Tightness of $\{\mu^\infty(t, dx , da )\}_{t \ge 0}$ implies that it is relative compact in the Prokhorov metric. This property, combined with the stability of \eqref{dqn_1}, yields $\{n(k)\}_{k \ge 0} \subset \{n\}_{n \ge 0}$, such that both $\lim_{n(k) \to \infty} \overline{\theta}(t_{n(k)})$ and $\lim_{n(k) \to \infty} \mu(t_{n(k)}, dx , da)$ have limits in $\Rb^d$ and $\Pit(\Sb \times \Ai)$, respectively. The properties of these limits, let us call them $\overline{\theta}^\infty$ and $\overline{\mu}^\infty$, determine the long-term behavior of \eqref{dqn_1}. Lemmas \ref{gp_bound} to \ref{ana_marginal_lemma} were stated and proved to build up to the most important result of this paper, which concerns the limiting behavior of \eqref{dqn_1}. We state and prove this result below, followed by a discussion of its implications.

\begin{theorem}\label{main_thm}
Assuming (A1)--(A5), the limit $\overline{\theta}^\infty$ of the deep Q-learning algorithm, i.e., iteration \eqref{dqn_1}, is such that
$\losst(\overline{\theta}^\infty, \overline{\mu}^\infty) = 0$ and
$\overline{\mu}^\infty(dx \times \Ai)$ is a stationary distribution of the state Markov process $x$.
\end{theorem}
\begin{proof}
From previous lemmas we know that \eqref{dqn_1} tracks $\boldsymbol{\theta}$, a solution to the non-autonomous ODE $\dot{\theta}(t) = \losst (\theta(t), \mu^\infty (t))$. Further, there is a sample path dependent compact subset of $\mathbb{R}^d$, $\mathcal{K}$, such that $\boldsymbol{\theta}$ remains inside of it. This is because the algorithm is assumed to be stable, i.e., $\theta_n \in \mathcal{K} \ \forall {n \ge 0}$. To determine the limit of the algorithm, $\overline{\theta}^\infty$, we need $\lim_{t \to \infty} \boldsymbol{\theta}(t)$.

To analyze $\dot{\theta}(t) = \losst (\theta(t), \mu^\infty (t))$, we transform it into an autonomous ODE through the standard change of variables trick. For this, we define $s(t) \coloneqq \frac{t}{1 + t}$, then $\dot{s}(t)  = (1 - s(t))^2$ and $t = \frac{s(t)}{1 - s(t)}$. We get the following transformed autonomous ODE:
\begin{equation}\label{trans_ode_eq}
\begin{split}
(\dot{\theta}(t), \dot{s}(t)) = \\ \left(\losst \left(\theta(t), \mu^\infty \left(\frac{s(t)}{1 - s(t)} \right) \right), (1 - s(t))^2 \right).
\end{split}
\end{equation}
Before proceeding, we state the following useful theorem, paraphrased to suit our purpose:\\

\vspace*{0.1in}

\textbf{[Theorem 2, Chapter 6 of \cite{aubin2012differential}]} \textit{Let $F$ be a continuous map from a closed subset $\hat{\mathcal{K}} \subset \mathcal{X}$ to $\mathcal{X}$. Let $x(\cdotp)$ be a solution trajectory of $\dot{x}(t) = F(x(t))$, such that it is inside $\hat{\mathcal{K}}$. Then, the solution converges to $x^*$, an equilibrium of $F$.}

\vspace*{0.1in}

To utilize the theorem, we define the following: $\mathcal{X} \coloneqq \Rb^d \times [0,1]$, $\hat{\mathcal{K}} \coloneqq \mathcal{K} \times [0,1]$, and $F: \hat{\mathcal{K}} \to \mathcal{X}$ such that $F(\theta, s) \coloneqq \left(\losst \left(\theta, \mu^\infty \left(\frac{s}{1 - s} \right) \right), (1 - s)^2 \right)$. It now follows from the above theorem that the transformed ODE \eqref{trans_ode_eq} converges to $(\overline{\theta}^\infty, 1)$, an equilibrium of $F$. Further, $1$ is the unique equilibrium point of $(1 - s)^2$, and $\overline{\theta}^\infty$ is an equilibrium of $\losst(\overline{\theta}^\infty, \overline{\mu}^\infty)$, where $\lim_{t \to \infty} \mu^\infty(t) \overset{d}{\implies} \overline{\mu}^\infty$. We discussed the existence of the limit $\overline{\mu}^\infty$ in the paragraph before stating this theorem.

Lemma~\ref{ana_marginal_lemma} shows that $\mu^\infty(t)$ is a stationary distribution of the state Markov process $x$ for all $t \ge 0$, i.e., $$\mu^\infty (t, dy \times \Ai) = \int_{\Sb} p(dy \mid x, \Ai, \theta^\infty(t)) \ \mu^\infty (t, dx \times \Ai).$$
Letting $t \to \infty$ on both sides of the above equation yields
$$\overline{\mu}^\infty (dy \times \Ai) = \int_{\Sb} p(dy \mid x, \Ai, \overline{\theta}^\infty) \ \overline{\mu}^\infty (dx \times \Ai).$$
In other words, the marginal over the states, $\overline{\mu}^\infty(dx \times \Ai)$, is stationary with respect to the state process.
\end{proof}

\subsection{On practical implications of the theory}
The primary goal of Deep Q-Learning is to find the optimal DQN-weights $\theta^*$ such that $\argmax_{a \in \Ai} \ Q(x,a; \theta^*) = \argmax_{a \in \Ai} \ Q^*(x,a)$, where $Q^*$ is the optimal Q-function. This is achieved by minimizing the squared Bellman loss. Theorem~\ref{main_thm} states that the Deep Q-Learning algorithm given by \eqref{dqn_1} converges to $\overline{\theta}^\infty$, a local minimizer of the average squared Bellman loss. The averaging over state-action pairs is induced by the limiting measure $\overline{\mu}^\infty \in \Pit(\Sb \times \Ai)$. In particular, we have
\begin{equation} \label{avg_eq}
\int \loss(\overline{\theta}^\infty, x, a) \ \overline{\mu}^\infty (dx, da) \ = \ 0.
\end{equation}
Lemma~\ref{ana_marginal_lemma} states that the limiting marginal distribution $\overline{\mu}^\infty (dx \times \Ai)$, over the state space $\Sb$, is stationary. Deep Q-Learning is typically employed in complex environments with multiple stationary distributions. Since $\overline{\mu}^\infty$ captures the long-term behavior of the training process, it directly depends on the distribution of the data encountered during training. As the squared Bellman loss is minimized on average in accordance to $\overline{\mu}^\infty$, the quality of learning is entirely captured by $\overline{\mu}^\infty$. In particular, the trained DQN approximates the optimal Q-factors accurately for state-action pairs that are distributed in accordance to $\overline{\mu}^\infty$. Performance is therefore good when encountering states arising from the ``limiting marginal''.


Fix $a \in \Ai$ and let $\Sb(a)$ be a measurable subset of $\Sb$ such that $a$ is the optimal action associated with every $x \in \Sb(a)$. For the sake of illustration, we consider a scenario wherein $\overline{\mu}^\infty (\Sb(a) \times \Ai) > 0$ and $\overline{\mu}^\infty (\Sb(a) \times a) = 0$. Roughly speaking, the set of state-action pairs given by $\{(x,a) \mid x \in \Sb\}$ were not encountered during training. This could happen, for example, due to poor exploration-exploitation trade-offs, or due to improper initialization of the DQN weights. The Q-factors may hence be poorly approximated on $\Sb(a) \times a$, and the trained DQN-agent cannot be expected to take optimal actions in these states. This explains the observation that, in practice, Deep Q-Learning sometimes fails to generalize well beyond the data encountered during training. Existing literature (see e.g.\ \cite{yang2019theoretical, zou2019finite}) does not account for such behaviors. Since DQN is usually trained using a simulator, it may be possible to empirically estimate $\overline{\mu}^\infty$. This knowledge may help identify scenarios wherein DQN is undertrained, thereby avoiding circumstances like the one sketched above.

\section{Weakening (A4) to allow twice continuously differentiable non-squashing activation functions} \label{sec_activation}
The hitherto presented analysis accounts for DQN architectures with differentiable squashing activations. In this section, we discuss modifications to our analysis that allow for general activations as well. In particular, the modifications account for activations such as Sigmoid Linear Unit (SiLU), Gaussian Error Linear Unit (GELU), etc.

Let us begin by understanding the role of squashing activations in our analysis. In Lemma~\ref{gp_bound}, the squashing property is used to find a $x$-independent $\hat{L}$ such that $|Q(x,a; \theta)| \le \hat{L} \lVert \theta \rVert_2$. Note that Lemma~\ref{gp_bound} is true even when the activations are non-squashing, provided $\Sb$ is a compact metric space. Since (A2) states that $\sup \limits_{n \ge 0} \ \lVert x_n \rVert_2 < \infty$ a.s., there is a sample path dependent compact set $\Sb_c \subset \Sb$ such that $x_n \in \Sb_c$ $\forall \ n \ge 0$. Using this information, we may modify the statement of Lemma~\ref{gp_bound} as follows:
\begin{lemma} \label{gp_bound_modified}
$\forall \ \theta \in \Rb^d$ $\sup \limits_{ a \in \Ai} \left| Q(x,a; \theta) \right| \le \tilde{L} \lVert \theta \rVert_2$, and $\tilde{L} > 0$ is dependent on $x$. Further, there is a sample path dependent $\hat{L}$, independent of $x$, such that $\sup \limits_{x \in \Sb_c} \sup \limits_{ a \in \Ai} \left| Q(x,a; \theta) \right| \le \hat{L} \lVert \theta \rVert_2$, where $\Sb_c$ is as defined above.
\end{lemma}

Parts of the analysis using Lemma~\ref{gp_bound} must now be modified to use Lemma~\ref{gp_bound_modified}. Other Lemmata, for e.g., Lemma~\ref{grad_Qexp} do not change when using Lemma~\ref{gp_bound_modified} instead of Lemma~\ref{gp_bound}.

\section{Extension to account for experience replay} \label{sec_replay}
Now, we extend our analysis to account for experience replay, an idea that allows the RL agent to relearn from past experiences. Specifically, at time $T$, the agent has ready access to $\{(x_k,a_k,r(x_k ,a_k), x_{k+1})\}_{T-H+1 \le k \le T}$, the history of states encountered, actions taken, rewards received and transitions made. The optimal size of the experience replay $H$ is problem dependent, and tunable. At time $T$, to update the NN weights $\theta$, the agent first samples a mini-batch of size $\hat{H} < H$ from the experience replay and calculates the following average loss gradient:

\begin{equation*}
\begin{split}
 \frac{1}{\hat{H}} \sum \limits_{i=1}^{\hat{H}} \loss \left(\theta_T, x_{k(T,i)}, a_{k(T,i)} \right), \text{ where } \\ T-H+1 \le k(T, i) \le T.
\end{split}
\end{equation*}
The DQN weights are updated as follows:
\begin{equation} \label{dqn_1_mod}
\theta_{n+1} = \theta_n + \gamma(n) \left[  \frac{1}{\hat{H}} \sum \limits_{i=1}^{\hat{H}} \loss \left(\theta_n, x_{k(n,i)}, a_{k(n,i)} \right) \right].
\end{equation}
To analyze \eqref{dqn_1_mod}, we must redefine $\mu$. For $t \in [t_n, t_{n+1})$, redefine $\mu(t)$ to be the probability measure (on $\Sb \times \Ai$) that places a mass of $\nicefrac{1}{\hat{H}}$ on $(x_{k(n,i)}, a_{k(n,i)})$ for $1 \le i \le \hat{H}$. With the new definition of $\mu$, for $t = t_n$ we get: 
\begin{equation*}
\begin{split}
\losst (\ct(t), \mu(t)) = \int \loss (\ct(t), x, a) \ \mu(t) = \\ \frac{1}{\hat{H}} \sum \limits_{i=1}^{\hat{H}} \loss \left(\theta_n, x_{k(n,i)}, a_{k(n,i)} \right).
\end{split}
\end{equation*}
Emulating the proofs of the Lemmata up to Lemma~\ref{ana_ode} for the new $\mu$, shows that \eqref{dqn_1_mod} tracks a solution to the non-autonomous o.d.e. $\dot{\theta}(t) = \losst(\theta(t), \mu^\infty(t))$. Again, $\mu^\infty$ is a limit of the redefined measure process sequence $\{\mu([t,\infty))\}_{t \ge 0}$ in $\mathcal{U}$.

Lemma~\ref{ana_marginal_lemma} states the the limiting marginal measure process $\mu^\infty(t, dx \times \Ai)$ is stationary with respect to the state Markov process for every $t \ge 0$. For it to hold in the presence of experience replay we redefine $\xi_n$ and $\mathcal{F}_{n}$ as follows:
\begin{equation*}
\begin{split}
\xi_n \coloneqq \sum \limits_{m=0}^{n-1} \frac{1}{\hat{H}} \left[\sum \limits_{i=1}^{\hat{H}} ( f(x_{k(m,i) + 1}) - \right.\\ \left. \int f(y) p(dy \mid x_{k(m,i)}, a_{k(m,i), \theta_{k(m,i)}}) ) \right],
\end{split}
\end{equation*}
$\mathcal{F}_{n-1} = \sigma \left\langle x_m, a_m, \theta_m, \Xi_m \mid m \le n \right\rangle$ for $n \ge 1$, where $\{\Xi_n \}_{n \ge 0}$ is the random process associated with mini-batch sampling. Typically the mini-batches are all sampled independently over time, hence $\{\Xi_n \}_{n \ge 0}$ constitutes an independent sequence of random variables. With these modifications the rest the steps involved in the proof of Lemma~\ref{ana_marginal_lemma} may be readily emulated. This would directly lead to the statement of the main result, Theorem~\ref{main_thm}. In conclusion, Deep Q-Learning with experience replay, \eqref{dqn_1_mod}, converges to $\hat{\theta}^\infty$ such that $\loss(\hat{\theta}^\infty, \hat{\mu}^\infty) = 0$, where $\hat{\mu}^\infty$ is a limit of $\{\tilde{\mu}^\infty(t)\}_{t \ge 0}$ as $t \to \infty$, and $\tilde{\mu}^\infty$ is the limiting measure process of the redefined $\mu$-process. Again, $\hat{\mu}^\infty (dx \times \Ai)$ is stationary with respect to the state Markov process. 

It is a common belief among deep learning practitioners that experience replay plays an important role in stabilizing the DQN training. In regards to the long-term behavior, we show that the use of experience replay has a qualitative effect on learning. \emph{This is because the limiting measure $\tilde{\mu}^\infty$ is shaped by the mini-batches sampled from experience replay during training, and it is richer than the one resulting from no experience replay.}

\section{CONCLUSION}
In this paper, we presented an asymptotic analysis of Deep Q-Learning under practical and verifiable assumptions.
An important contribution is the complete characterization of the DQN performance as a function of training. We obtained this result by analyzing the limit of a closely associated measure process (on the state-action pairs). The result has various implications that we shall elaborate on more closely in future work. In particular, is helps explain empirical observations regarding the performance of Deep Q-Learning that current theory does not account for. Practically motivated extensions and generalizations like this one are also on our agenda of future work.

\bibliographystyle{plain}
\bibliography{references}

\section{Appendix: Preliminaries: Reinforcement Learning (RL)} \label{sec:preliminaries_RL}

In RL an \emph{agent} interacts with an \emph{environment} over time, via \emph{actions}. It takes the current (environment) \emph{state} into consideration to pick an action, and receives a feedback in terms of a \emph{reward}. The environment then moves to a new state. This is schematically represented in Figure~\ref{RL_MDP}. \emph{The goal in RL is to ensure that the agent takes a sequence of actions, such that the rewards accumulated over time are maximized}. 
\begin{figure}[h]
\begin{center}
\includegraphics[width=.48 \textwidth]{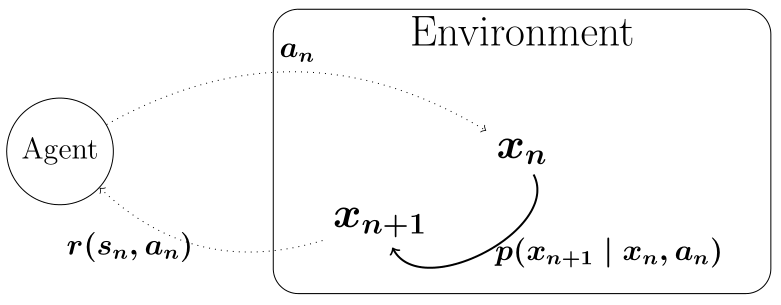}
\caption{Snapshot of interaction at step $n$}
\label{RL_MDP}
\end{center}
\end{figure}

Formally speaking, the above stated interactions can be modelled as a Markov Decision Process (MDP). It is defined as a $5$-tuple $(\Sb, \Ai, p, r, \alpha)$, where:
\begin{itemize}
\item[$\Sb$] is the state space. In typical applications $\Sb \equiv \Rb^k$, $k >0$. 
\item[$\Ai$] is the action space. In this paper, $\Ai$ is a discrete finite set.
\item[$p$] is the ``controlled'' transition kernel. We use $p(\cdotp \mid x,a)$ to represent the distribution of the next state given the current state and action.
\item[$r$] is the reward function. In particular, $r(x,a)$ denotes the reward associated with taking action $a$ at state $x$.
\item[$\alpha$] is the discount factor with $0 < \alpha \le 1$. It is used to discount the relevance of future consequences of actions.
\end{itemize}
A policy $\pi$ is defined as a function from $\Sb$ to $\Ai$. Given $\pi$, we can associate a \emph{Value function} $V^\pi(x)$ with each $x \in \Sb$, with $V^\pi(x) \coloneqq \mathbb{E} \left[ \sum \limits_{n \ge 0} \alpha^n r(x_n, \pi(x_n)) \mbox{\Large $\vert$} x_0 = x \right]$. The goal in RL can be restated to find $\pi^*$ such that $V^{\pi^*} (x) = \max \limits_{\pi} V^\pi (x)$ for all $x \in \Sb$. In Dynamic Programming parlance $\pi^*$ is a solution to the \emph{infinite horizon discounted reward problem}. 

Closely related to the value function is the concept of Q-function, defined over state-action pairs $(x,a) \in \Sb \times \Ai$ by
$
Q^{\pi}(x,a) \coloneqq r(x,a) + \alpha \int V^{\pi}(x') \ p(dx' \mid x,a),
$
where $\pi$ is a fixed policy. The optimal Q-function is defined as: $$Q^*(x,a) \coloneqq r(x,a) + \alpha \int V^{\pi^*}(x') \ p(dx' \mid x,a).$$
Clearly, $\max \limits_{a \in \Ai} Q^*(x,a) = V^{\pi^*}(x)$ and $\pi^*(x) = \underset{a \in \Ai}{\argmax}\ Q^*(x,a)$ for all $x \in \Sb$. Hence, in order to find $\pi^*$ it is sufficient to find $Q^*$. This is the idea behind Q-Learning. Its variant, Deep Q-Learning, has shown tremendous promise in solving complex problems involving continuous state spaces, where Q-Learning typically fails. It involves parameterizing the optimal Q-function using a DNN, called the Deep Q-Network (DQN). The goal is to find the optimal set of parameters (DQN weights) $\theta^*$, \emph{by interacting with the environment}, such that $Q(x,a; \theta^*) \approx Q^*(x,a)$ for all $(x,a) \in \Sb \times \Ai$. The DQN is trained to minimize the following squared Bellman loss over all state-action pairs $(x,a)$: $$\left[r(x,a) + \alpha \int \max \limits_{a' \in \Ai} Q(x',a'; \theta) \ p(dx' \mid x,a) -  Q(x,a; \theta) \right]^2.$$


\section{Appendix:
Technical lemmas supporting Lemma~\ref{grad_loss_continuity}} \label{sec:aux_lemmas}
Let us recall that every action is associated with a different output layer: $Q(x,a; \theta) = \sum \limits_{i=1}^{l(a)} \act_a(i) \theta_a(i)$, with $l(a)$ the width of the layer associated with action $a$.
\begin{lemma} \label{gp_bound}
$\sup \limits_{x \in \Sb, \ a \in \Ai} \left| Q(x,a; \theta) \right| \le \hat{L} \lVert \theta \rVert_2$, for some $\hat{L} > 0$.
\end{lemma}
\begin{proof}
We begin by noting that activation functions considered herein are also squashing. Hence, absolute values of their outputs are bounded by some $0 < c < \infty$.
Let us fix arbitrary $x \in \Sb$ and $a \in \Ai$, then $$\left| Q(x,a; \theta)\right| \le c \sum \limits_{i=1}^{l(a)} \left| \theta_a(i) \right|  =  c \lVert \theta_a \rVert_1,$$ where $\lVert \cdotp \rVert_1$ is the 1-norm. It now follows from $\lVert \theta_a \rVert_1 \le l(a) \lVert \theta_a \rVert_2$, that
$\left| Q(x,a; \theta)\right| \le c l(a)  \lVert \theta_a \rVert_2$. If we let $\hat{L} \coloneqq c\ l(a)$, then the statement of the lemma follows.
\end{proof}

Since we allow for possibly unbounded $Q$-factors, the above lemma indicates that we need arbitrarily large DQN weights for good approximation. Depending on the system states encountered during training, the Deep Q-Learning algorithm explores an appropriate subspace associated with the weight vector. Hence, the approximation capability of the trained DQN depends on the state-action pairs encountered during training. The difference, in state distributions, between the training and test scenarios will determine performance.

Recall that we parameterize the $Q$-function using a neural network that consists of twice continuously differentiable activation functions. Hence, $Q$ may be viewed as a composition of twice continuously differentiable activations, and the DQN weight vector. In other words, $Q$ itself is twice continuously differentiable. This intuition is formalized in the next lemma.

\begin{lemma} \label{grad_2cd}
$Q(x, a; \theta)$ is twice continuously differentiable in the $\theta$-coordinate for every $x \in \Sb$ and $a \in \Ai$, where $\theta$ is the DQN weight vector.
\end{lemma}
\begin{proof}
Recall that the DQN weights are updated using the back propagation algorithm, i.e., the chain rule. Given the DQN weight-vector $\theta \in \Rb^d$, we need to show that $\Large{ \nicefrac{\partial ^2 Q(\hat{x}, \hat{a}; \theta)}{\partial \theta_i ^2} }$ exists and is continuous for $1 \le i \le d$. Also, recall from the \emph{note on tunable biases} at the end of Section~\ref{DQN_prelim}, that without loss of generality we may only consider tunable edge weights, and ignore tunable bias terms.

Let us fix an arbitrary $\hat{x} \in \Sb$, $\hat{a} \in \Ai$ and $i \in \{1, \ldots, d\}$. DQN weight $\theta_i$ is associated with an edge of the NN. Also associated with this edge is another weight $\ei{e}_i \coloneqq \act_i \theta_i$, where $\act_i$ is the output of an activation from the previous layer (from the head of the edge). This is illustrated in Fig.~\ref{grad_nn}.

\begin{figure}[h]
\begin{center}
\includegraphics[width=.48 \textwidth]{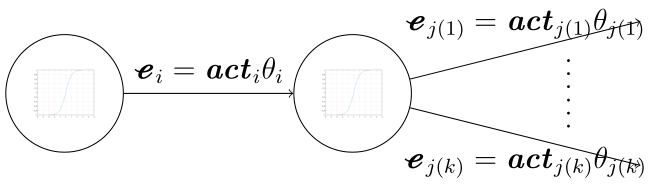}
\caption{Section of a DNN}
\label{grad_nn}
\end{center}
\end{figure}



To prove the lemma, we show something stronger, i.e., that both ${\Large \nicefrac{\partial ^2 Q(\hat{x}, \hat{a}; \theta)}{\partial \theta_i ^2} }$ and ${\Large \nicefrac{\partial ^2 Q(\hat{x}, \hat{a}; \theta)}{\partial \ei{e}^2}}$ are continuous. The proof involves inducting on the depth of the DNN, starting from the output layer, and going backwards. Note that $Q(\hat{x}, \hat{a}; \theta)  = \sum \limits_{i=1}^{l(\hat{a})} \ei{e}_{\hat{a}(i)}$, where $\ei{e}_{\hat{a}(i)} \coloneqq \boldsymbol{act}_{\hat{a}(i)}  \ \theta_{\hat{a}(i)}$, where $l(\hat{a})$ is the number of activations in the output layer of action $\hat{a}$. Also, note that $\act_{\hat{a}(i)}$ is the output of the $i$-th activation in the output layer associated with action $\hat{a}$, and $\theta_{\hat{a}(i)}$ is the corresponding network edge-weight, $1 \le i \le l(\hat{a})$, see Section~\ref{DQN_prelim} for details. We have, ${\Large \nicefrac{\partial ^2 Q(\hat{x}, \hat{a}; \theta)}{\partial \theta_{a(i)}^2} } = {\Large \nicefrac{\partial ^2 Q(\hat{x}, \hat{a}; \theta)}{\partial \ei{e}_{a(i)} ^2}} = 0$ for all $a \neq \hat{a}$, where \emph{subscript $a(i)$ is used to indicate that $\theta_{a(i)}$ and $\ei{e}_{a(i)}$ are associated with the output layer of action $a$}. Twice continuous differentiability with respect to $\theta_{\hat{a}(i)}$ and $\ei{e}_{\hat{a}(i)}$ directly follows from the same property of the activation units, $1 \le i \le l(\hat{a})$. 

Let us assume that the hypothesis is true for weights associated with edges out of the $(l+1)^{st}$ layer and prove for the $l^{th}$ layer. Fig.~\ref{grad_nn} illustrates an edge out of an $l^{th}$ layer activation, and its associated weight $\ei{e}_i \coloneqq \act_i \ \theta_i$, where $i \in \{1, \ldots, d\}$. Also note that, in the Fig.~\ref{grad_nn}, $\act_{j(m)} = \act_0$ for all $1 \le m \le k $. It follows directly from the back-propagation algorithm (chain rule) that:
$$
\frac{\partial Q(\hat{x}, \hat{a}; \theta)}{\partial \ei{e}_i} = \frac{\partial \act_0}{\partial \ei{e}_i} \sum \limits_{m=1}^k \  \left[ \frac{\partial Q(\hat{x}, \hat{a}; \theta)}{\partial \ei{e}_{j(m)}} \theta_{j(m)} \right],
$$
\begin{equation*}
\begin{split}
\frac{\partial^2 Q(\hat{x}, \hat{a}; \theta)}{\partial \ei{e}_i^2} = \left(\frac{\partial \act_0}{\partial \ei{e}_i} \right)^2 \sum \limits_{m=1}^k \  \left[ \frac{\partial^2 Q(\hat{x}, \hat{a}; \theta)}{\partial \ei{e}_{j(m)}^2} \theta_{j(m)} ^2 \right] + \\ \frac{\partial^2 \act_0}{\partial \ei{e}_i^2} \sum \limits_{m=1}^k \  \left[ \frac{\partial Q(\hat{x}, \hat{a}; \theta)}{\partial \ei{e}_{j(m)}} \theta_{j(m)} \right].
\end{split}
\end{equation*}
From the induction hypothesis and the twice continuous differentiability of $\act_0$, we get that $\frac{\partial^2 Q(\hat{x}, \hat{a}; \theta)}{\partial \ei{e}_{i}^2}$ is continuous. Next, we observe: 
$$
\frac{\partial Q(\hat{x}, \hat{a}; \theta)}{\partial \theta_{i}} = \frac{\partial Q(\hat{x}, \hat{a}; \theta)}{\partial \ei{e}_{i}} \frac{\partial \ei{e}_{i}}{\partial \theta_{i}} = \act_{i} \frac{\partial Q(\hat{x}, \hat{a}; \theta)}{\partial \ei{e}_{i}},
$$

$$
\frac{\partial^2 Q(\hat{x}, \hat{a}; \theta)}{\partial \theta_{i}^2} = \left( \act_{i} \right)^2 \frac{\partial^2 Q(\hat{x}, \hat{a}; \theta)}{\partial \ei{e}_{i}^2}.
$$
The continuity of {\Large $\frac{\partial^2 Q(\hat{x}, \hat{a}; \theta)}{\partial \theta_{i}^2}$} follows from the twice continuous differentiability of {\Large $\frac{\partial^2 Q(\hat{x}, \hat{a}; \theta)}{\partial \ei{e}_{i}^2}$}. 
\end{proof}

Since $Q$ is two times continuously differentiable in the $\theta$-coordinate, it is locally Lipschitz continuous in that coordinate. Also, the Lipschitz constant may depend on $x$, in addition to $\theta$. Let us fix arbitrary $\hat{a} \in \Ai$ and $\hat{\theta} \in \Rb^d$. Since $Q(\cdotp, \hat{a}; \hat{\theta})$ and $\nabla_\theta Q(\cdotp, \hat{a}; \hat{\theta})$ are composed (via addition and multiplication) of twice continuously differentiable functions (activation units), we get that both $Q$ and $\nabla_\theta Q$ are continuous in the $x$-coordinate. Although we do not need it here, the stronger property of local Lipschitz continuity may also be shown. Finally, note that $Q$ and $\nabla_\theta Q$ are continuous in the $a$-coordinate, since $\Ai$ is finite.

\begin{lemma} \label{grad_Qexp}
The following map is continuous and locally Lipschitz continuous in the $\theta$-coordinate:
$$
(x,a,\theta) \mapsto \int \max \limits_{a \in \Ai} Q(x',a;\theta) \ p(dx' \mid x,a,\theta) \, .
$$

\end{lemma}
\begin{proof}
We begin by fixing arbitrary $\hat{x} \in \Sb$ and $\hat{a} \in \Ai$. Given $\theta \in \Rb^d$, Lemma~\ref{grad_2cd} implies the existence of $\mathcal{N}(\theta, \hat{x})$, without loss of generality a compact neighborhood of $\theta$, and $L(\theta, \hat{x}) > 0$, such that $\forall \ \theta_1, \theta_2 \in \mathcal{N}(\theta, \hat{x})$:
$$
\left| Q(\hat{x}, \hat{a}; \theta_1) - Q(\hat{x}, \hat{a}; \theta_2) \right| \le L(\theta, \hat{x}) \lVert \theta_1 - \theta_2 \rVert_2.
$$
Since $\hat{a}$ is fixed, $p(\cdotp \mid \hat{x}, \hat{a}, \theta) \equiv p(\cdotp \mid \hat{x}, \hat{a})$, i.e., the transition kernel does not depend on $\theta$. Recall that the dependence of $p$ on $\theta$ is only via the action $a$. Define $a_1(x) \coloneqq \underset{a \in \Ai}{\text{argmax}} \ Q(x, a; \theta_1)$, then following the above line of thought (with ``$x$'' replacing ``$\hat{x}$'' and ``$a_1(x)$'' replacing ``$\hat{a}$'') we get:
\begin{equation} \label{grad_3_1}
\begin{split}
& \left| \max \limits_{a \in \Ai} Q(x, a; \theta_1)  - \max \limits_{a \in \Ai} Q(x, a; \theta_2) \right| \\ 
& \qquad \le \left| Q(x, a_1(x); \theta_1) -  Q(x, a_1(x); \theta_1) \right| \\
& \qquad \le L(\theta, x) \, \lVert \theta_1 - \theta_2 \rVert.
\end{split}
\end{equation}
Hence, from Lemma~\ref{gp_bound} and the compactness of $\mathcal{N}(\theta, \hat{x})$, we conclude that
$$
\sup \limits_{\hat{\theta} \in \mathcal{N}(\theta, \hat{x})} \ \sup \limits_{x \in \Sb} \ \sup \limits_{a \in \Ai} \left| Q(x,a; \hat{\theta}) \right| < \infty.
$$
In particular, there exists a bounded measurable function $\hat{F}_\theta: x \mapsto L(x, \theta)$ such that \eqref{grad_3_1} is satisfied for every $x \in \Sb$, with $\hat{F}_\theta (x)$ as the Lipschitz constant.
Hitherto presented arguments and observations yield:
\begin{equation}
\begin{split}
& \left| \int \max \limits_{a \in \Ai} Q(x, a; \theta_1) p(dx \mid \hat{x}, \hat{a}, \theta_1) - \right. \\ 
& \qquad \left. \int \max \limits_{a \in \Ai} Q(x, a; \theta_2) p(dx \mid \hat{x}, \hat{a}, \theta_2) \right| \le \\ & \quad \leq \lVert \theta_1 - \theta_2 \rVert_2 \int L(\theta, x) p(dx \mid \hat{x}, \hat{a}) \le L \lVert \theta_1 - \theta_2 \rVert_2,
\end{split}
\end{equation}
where $L = 2 \times \sup \limits_{\hat{\theta} \in \mathcal{N}(\theta, \hat{x})} \ \sup \limits_{x \in \Sb} \ \sup \limits_{a \in \Ai} \left| Q(x,a; \hat{\theta}) \right|$.

Let us fix arbitrary $\hat{\theta} \in \Rb^d$ and $\hat{a} \in \Ai$. Define $\hat{a}(x) \in \underset{a \in \Ai}{\text{argmax}}\ Q(x,a; \hat{\theta})$ and $\hat{Q}(x) \coloneqq Q(x, \hat{a}(x), \hat{\theta})$ for all $x \in \Sb$. Note that there may be many actions that maximize the $Q$ function, $\hat{a}(\cdotp)$ selecting one of them. First, we show that $x_n \to x$ implies that $\hat{Q}(x_n) \to \hat{Q}(x)$, and hence that $\hat{Q} \in \Cb _b(\Sb)$ (from Lemma~\ref{gp_bound}). To this end, we show that every subsequence of $\{\hat{Q}(x_n)\}_{n \ge 0}$ has a further subsequence that converges, and the limit always equals $\hat{Q}(x)$. Let us begin by considering the entire sequence itself. Since $\Ai$ is a compact metric space, $\exists \ \{n(m)\}_{m \ge 0} \subset \{n \}_{n \ge 0}$ such that $\hat{a}(x_{n(m)}) \to \hat{a}$ for some $\hat{a} \in \Ai$, hence $Q(x_{n(m)}, \hat{a}(x_{n(m)})); \hat{\theta}) \to Q(x, \hat{a}; \hat{\theta})$. We claim that $\hat{a} = \hat{a}(x)$, thus implying $\hat{Q}(x_{n(m)}) \to \hat{Q}(x)$.
To see that the claim is true assume the contrary. In other words, $\hat{a}(x) \neq \hat{a}$ and $Q(x, \hat{a}(x); \hat{\theta}) > Q(x, \hat{a}; \hat{\theta}) + \epsilon$, for some $\epsilon > 0$. From the continuity of $Q$, we get that $\exists \ M > 0$ with $\left|Q(x_{n(m)}, \hat{a}(x); \hat{\theta}) - Q(x, \hat{a}(x); \hat{\theta})  \right| \le \nicefrac{\epsilon}{4}$ and $\left| Q(x_{n(m)}, \hat{a}(x_{n(m)}); \hat{\theta}) - Q(x, \hat{a}; \hat{\theta})  \right| \le \nicefrac{\epsilon}{4}$, for all $m \ge M$. Hence, we get that $Q(x_{n(m)}, \hat{a}(x); \hat{\theta}) > Q(x_{n(m)}, \hat{a}(x_{n(m)}); \hat{\theta})$, a contradiction. Finally, we note that the above set of arguments can be repeated starting with any subsequence of $\{n\}_{n \ge 0}$.

Now that we have $\hat{Q} \in \Cb _b(\Sb)$, we are ready to prove continuity in the $x$-coordinate. Recall that we have assumed the transition kernel to be continuous in $x$. Hence $x_n \to x$ implies that $p(\cdotp \mid x_n, \hat{a}, \hat{\theta}) \overset{d}{\implies} p(\cdotp \mid x, \hat{a}, \hat{\theta})$, i.e., the kernels converge in distribution. It now follows from the definition of ``convergence in distribution'' that $\int \hat{Q}(y) p(dy \mid x_n, \hat{a}, \hat{\theta})  \to \int \hat{Q}(y) p(dy \mid x, \hat{a}, \hat{\theta})$. In other words, we have the required, namely 

\begin{equation*}
\begin{split}
x_n \to x \implies \int \max \limits_{a \in \Ai}Q(y, a, \hat{\theta}) p(dy \mid x_n, \hat{a}, \hat{\theta})  \to \\  \int \max \limits_{a \in \Ai} Q(y, a, \hat{\theta}) p(dy \mid x, \hat{a}, \hat{\theta})
\end{split}
\end{equation*}
as $n \to \infty$ 
Finally, recall that $\Ai$ is compact metrizable as it is a finite. Hence continuity in the $a$-coordinate is trivial.
\end{proof}



\section{Appendix: Missing Proofs}\label{sec:missing_proofs_2}
\subsection{Proof of Lemma~\ref{lemma:trajectory_track}} \label{sec:trajectory_track}
\begin{proof}
First we define the notation $[t]$ for $t \ge 0$ as $[t] \coloneqq t_{\sup \{n \mid t_n \le t\}}$. Next, we need to show that: $$
\sup \limits_{t \in [0,T]} \lVert \ct(t_n +t) - \ct([t_n + t]) \rVert \in \Theta(\gamma(n)).
$$
For this, we fix $t \in [0,T]$, then $[t_n + t] = t_{n+k}$ for some $k \ge 0$. Recall that 
$$
\ct(t_n + t) = \ct(t_{n+k}) + \frac{t_n + t - t_{n+k}}{\gamma(n+k)} \left( \ct(t_{n+k+1}) - \ct(t_{n+k}) \right).
$$
We use the following: $\left\lVert \ct(t_{n+k+1}) - \ct(t_{n+k}) \right\rVert \le \gamma(n+k) \lVert \loss(\ct(t_{n+k}), x_{n+k}; a_{n+k}) \rVert$; the stability of the algorithm, i.e., (A2); the monotonic property of the step-size sequence, i.e., (A1); and the boundedness of $\loss$ as a function of $\theta$, to obtain $\lVert \ct(t_n + t) - \ct(t_{n+k}) \rVert \in \Theta(\gamma(n))$. Similarly, let us show that:
$$
\sup \limits_{t \in [0,T]} \lVert \theta^n(t) - \theta^n ([t_n + t] - t_n) \rVert \in \Theta(\gamma(n)).
$$
Again, $[t_n + t] = t_{n+k}$ for some $k \ge 0$. We also have $\lVert \theta^n(t) - \theta^n (t_{n+k} - t_n) \rVert  = \left\lVert \int \limits_{t_{n+k} - t_n}^t \losst(\theta(s), \mu^n(s)) \ ds \right\rVert$. Using arguments similar to the ones made before, the required statement directly follows. It follows from all of the above arguments that it is enough to show the following in order to prove the lemma:
$$
 \sup \limits_{t \in [0,T]} \lVert \ct([t_n + t]) - \theta^n([t_n + t] -  t_n) \rVert \to 0.
$$
Once again we let $[t_n + t] = t_{n+k}$ for some $k \ge 0$, and observe that 
\begin{equation*}
\begin{split}
\lVert \ct([t_n + t]) - \theta^n([t_n + t] -  t_n) \rVert \le \\ \sum \limits_{m=n}^{n+k-1} \int \limits_{t_m}^{t_{m+1}} \lVert \losst(\ct([s]), \mu^n(s-t_n)) - \\ \losst(\theta^n (s - t_n), \mu^n(s - t_n))) \rVert \ ds,
\end{split}
\end{equation*}
\begin{equation*}
\begin{split}
\lVert \ct([t_n + t]) - \theta^n([t_n + t] -  t_n) \rVert \le \\ \sum \limits_{m=n}^{n+k-1} \int \limits_{t_m}^{t_{m+1}} L \left\lVert \ct([s]) - \theta^n (s - t_n) \right\rVert.
\end{split}
\end{equation*}
Adding and subtracting $\theta^n([s] - t_n)$, the R.H.S. of above equation is less than or equal to
\begin{equation*}
\begin{split}
\sum \limits_{m=n}^{n+k-1} L \int \limits_{t_m}^{t_{m+1}} \left\lVert \theta^n (s - t_n) - \theta^n ([s] - t_n) \right\rVert
+ \\ \sum \limits_{m=n}^{n+k-1} L \int \limits_{t_m}^{t_{m+1}} \left\lVert \ct([s]) - \theta^n ([s] - t_n) \right\rVert.
\end{split}
\end{equation*}
Considering that $\lVert \ct(t_n + t) - \ct(t_{n+k}) \rVert$ and $\lVert \theta^n(t) - \theta^n ([t_n + t] - t_n) \rVert \in \Theta(\gamma(n))$, we get $\sum \limits_{m=n}^{n+k-1} \int \limits_{t_m}^{t_{m+1}} \left\lVert \theta^n (s - t_n) - \theta^n ([s] - t_n) \right\rVert \le \sum \limits_{m=n}^{n+k-1} \Theta(\gamma(m)^2)$, which goes to zero as $n \to \infty$.
Now we use the discrete version of Gronwall's inequality to get:
\begin{equation*}
\begin{split}
\lVert \ct([t_n + t]) - \theta^n([t_n + t] -  t_n) \rVert \le \\ \left( L \sum \limits_{m=n}^{n+k-1} \Theta(\gamma(m))^2 \right) \exp(LT).
\end{split}
\end{equation*}
\end{proof}

\subsection{Proof of Lemma~\ref{ana_marginal_lemma}} \label{sec_marginal_lemma}
\begin{proof}
Pick $f$ from $\Cb_{b}(\Sb)$, the convergence determining class for $\Pit(\Sb)$. Without loss of generality, we assume that $0 \le f \le 1$. We define the following zero mean Martingale with respect to the filtration $\mathcal{F}_{n-1} \coloneqq \sigma \left\langle x_m, a_m, \theta_m \mid m \le n \right\rangle$, for $n \ge 1$:
\begin{equation} \label{eq_l41}
\xi_n \coloneqq \sum \limits_{m=0}^{n-1} \gamma(m) \left[ f(x_{m+1}) - \int_\Sb f(y) p(dy \mid x_m, a_m, \theta_m) \right].
\end{equation}
Since $f$ is bounded and $\sum \limits_{n \ge 0} \gamma(n)^2 < \infty$, the quadratic variation process associated with the above Martingale is convergent. It follows from the Martingale Convergence Theorem \cite{durrett2010probability} that $\xi_n$ converges almost surely. Hence for $t > 0$,
\begin{equation} \label{eq_l42}
\sum \limits_{m=n}^{\tau(n,t)} \gamma(m) \left[ f(x_{m+1}) - \int_\Sb f(y) p(dy \mid x_m, a_m, \theta_m) \right] \to 0 \ a.s.,
\end{equation}
where $\tau(n,t) \coloneqq \min\{m \ge n \mid t_m \ge t_n + t\}$. Since the steps-sizes are eventually decreasing, hence $\sum \limits_{m=n}^{\tau(n,t)} [\gamma(m) - \gamma(m+1) ] f(x_{m+1}) \to 0$ a.s. Then \eqref{eq_l42} becomes:
\begin{equation} \label{eq_l43}
\sum \limits_{m=n}^{\tau(n,t)} \gamma(m) \left[ f(x_m) - \int_\Sb f(y) p(dy \mid x_m, a_m, \theta_m) \right] \to 0 \ a.s.
\end{equation}
Using the definition of $\mu$, we rewrite \eqref{eq_l43} as:
\begin{equation} \label{eq_l44}
\begin{split}
\int_{t_n}^{t_n + t} \int_{\Sb \times \Ai} \left[ f(x) - \int_\Sb f(y) p(dy \mid x,a, \ct(s)) \right] \mu(s, dx,da) ds \\ \to 0 \ a.s.
\end{split}
\end{equation}
Let us define a new function $\hat{f}(x,a) \coloneqq f(x)$ for all $(x,a) \in \Sb \times \Ai$, then $\hat{f} \in \Cb_b (\Sb \times \Ai)$. Since $\mu(t_n + \cdotp) \to \mu^\infty(\cdotp)$ in $\Ui$, it follows that as $n \to \infty$:
\begin{equation} \label{eq_l45}
\begin{split}
\int_{t_n}^{t_n + t} \int_{\Sb \times \Ai} \hat{f}(x,a) \mu(s, dx, da) ds \to \\ \int_{0}^{ t} \int_{\Sb \times \Ai} \hat{f}(x,a) \mu^\infty(s, dx, da) ds.
\end{split}
\end{equation}
Further, the limit in \eqref{eq_l45} equals $\int_{0}^{ t} \int_{\Sb} f(x) \mu^\infty(s, dx \times \Ai) ds$.

Recall that $(x,a,\theta) \mapsto p(\cdotp \mid x,a,\theta)$ is a continuous map. Since $f$ is a convergence determining function in $\Pit(\Sb)$, it follows that $\int _\Sb f(y) p(dy \mid x,a,\ct(s)) \to \int _\Sb f(y) p(dy \mid x,a,\theta^\infty(s))$ for all $s \in [0,t]$. Define $h_n(s,x,a) \coloneqq \int _\Sb f(y) p(dy \mid x,a,\ct(t_n + s))$ and $h_\infty(s,x,a) \coloneqq \int _\Sb f(y) p(dy \mid x,a,\theta^\infty(s))$. For a fixed $s \in [0,t]$, $h_n(s, \cdotp)$, $n \ge 0$, and $h_\infty(s, \cdotp)$ belong to $\Cb_b(\Sb \times \Ai)$. Hence,
\begin{equation}
\begin{split}
\int _{\Sb \times \Ai} h_n(s, x,a) \mu(t_n + s,dx,da) \to \\ \int _{\Sb \times \Ai} h_\infty(s, x,a) \mu^\infty(s,dx,da).
\end{split}
\end{equation}
It then follows from Dominated Convergence Theorem (DCT) \cite{durrett2010probability} that:
\begin{equation}
\begin{split}
\int_{t_n}^{t_n + t}  \int _{\Sb \times \Ai} h_n(s, x,a) \mu(s,dx,da) ds \to \\ \int_{0}^{t}  \int _{\Sb \times \Ai} h_\infty(s, x,a) \mu^\infty(s,dx,da) ds.
\end{split}
\end{equation}
In other words, we have
\begin{multline} \label{eq_l46}
\int_{t_n}^{t_n + t}  \int _{\Sb \times \Ai} \int _\Sb f(y) p(dy \mid x,a,\ct(s)) \mu(s,dx,da) ds \to \\ \int_{0}^{t}  \int _{\Sb \times \Ai} \int _\Sb f(y) p(dy \mid x,a,\theta^\infty(s)) \mu^\infty(s,dx,da) ds.
\end{multline}
From \eqref{eq_l44}, \eqref{eq_l45} and \eqref{eq_l46} we get:
\begin{equation} \label{eq_l47}
\begin{split}
\int_{0}^{ t} \int_{\Sb \times \Ai} f(x) \mu^\infty(s, dx, da) ds = \\ \int_{0}^{t}  \int _{\Sb \times \Ai} \int _\Sb f(y) p(dy \mid x,a,\theta^\infty(s)) \mu^\infty(s,dx,da) ds.
\end{split}
\end{equation}
Using Lebesgue's theorem we get that a.e. on [0,t]:
\begin{equation*}
\begin{split}
\int_{\Sb \times \Ai} f(x) \mu^\infty(s, dx, da) = \\  \int _{\Sb \times \Ai} \int _\Sb f(y) p(dy \mid x,a,\theta^\infty(s)) \mu^\infty(s,dx,da).
\end{split}
\end{equation*}
Applying Fubini's theorem \cite{durrett2010probability} to swap the double integral on the R.H.S. of the above equation, gives us:
\begin{equation*}
\begin{split}
\int_{\Sb} f(x) \mu^\infty(s, dx, \Ai) = \\ \int _{\Sb} f(y) \int _\Sb p(dy \mid x,\Ai,\theta^\infty(s)) \mu^\infty(s,dx,\Ai).
\end{split}
\end{equation*}
Since $f$ is a convergence determining function, we get that $\mu^\infty(s, dy, \Ai) = \int _\Sb p(dy \mid x,\Ai,\theta^\infty(s)) \mu^\infty(s,dx,\Ai)$. Hence, we have shown that the limiting distribution over the state-action pairs $\mu^\infty$ is such that, almost everywhere on $[0, \infty)$, its marginal over the state space constitutes a stationary distribution over the state Markov process with transition kernel $p(\cdotp \mid x, \Ai, \theta)$.

Now, it is left to show that the family of measures $\{\mu^\infty(t, dx , da )\}_{t \ge 0}$ is tight. From previous discussions and observations, given $t \ge 0$, we can find $\{n(m)\}_{m \ge 0} \subset \{n\}_{n \ge 0}$ such that
$$
\lim \limits_{n(m) \to \infty} \mu(t_{n(m)}, dx , da) \overset{d}{\implies} \mu^\infty(t, dx , da).
$$
Using the Portmanteau Theorem \cite{billingsley2013convergence}, we get $\mu^\infty(t, \mathcal{K} \times \Ai') \ge \limsup \limits_{n(m) \to \infty} \mu(t_{n(m)}, \mathcal{K} \times \Ai')$, where $\mathcal{K} \subset \Sb$ is compact and $\Ai' \subset \Ai$. Given $\epsilon > 0$, there exists $\mathcal{K}(\epsilon) \subset \Sb$, compact, such that $\inf \limits_{m \ge 0} \mu(t_{n(m)}, \mathcal{K}(\epsilon) 
\times \Ai') \ge 1 - \epsilon$ for any $\Ai' \subset \Ai$, as $\mu(t_{n(m)})_{m \ge 0}$ is tight. Hence $\mu^\infty(t, \mathcal{K}(\epsilon) \times \Ai') \ge 1 - \epsilon$. As $t$ was arbitrary, we get that $\{\mu^\infty(t, dx , da )\}_{t \ge 0}$ is tight. 
\end{proof}

\end{document}